\pdfoutput=1

\documentclass[11pt]{article}

\usepackage[]{EMNLP2022}

\usepackage{times}
\usepackage{latexsym}

\usepackage[T1]{fontenc}

\usepackage[utf8]{inputenc}

\usepackage{microtype}

\usepackage{inconsolata}

\usepackage{graphicx}
\graphicspath{{./figures/}}
\usepackage{amssymb}
\usepackage{amsmath}
\usepackage{amsthm}
\usepackage{url}
\usepackage{hyperref}
\usepackage{tabu}
\usepackage{longtable}
\newcommand{\Reals}{\mathbb{R}}
\definecolor{CBred}{RGB}{203,24,29}
\definecolor{CBgreen}{RGB}{35,139,69}
\definecolor{CBblue}{RGB}{39,130,189}

\newtheorem{theorem}{Theorem}
\newtheorem{lemma}{Lemma}
\newtheorem{cor}{Corollary}

\usepackage{dblfloatfix}

\title{Better Hit the Nail on the Head than Beat around the Bush:\\ Removing Protected Attributes with a Single Projection}

    \author{Pantea Haghighatkhah$^\diamondsuit$ \ \ Antske Fokkens$^{\clubsuit\diamondsuit}$\ \  Pia Sommerauer$^\clubsuit$ 
    \\ \textbf{Bettina Speckmann}$^\diamondsuit$ \ \ \textbf{Kevin Verbeek}$^\diamondsuit$\\
    \\
    $\diamondsuit$TU Eindhoven, Department of Mathematics and Computer Science\\
   $\clubsuit$ Vrije Universiteit Amsterdam, Computational Linguistics Text Mining Lab\\
  }

\begin{document}
\maketitle
\begin{abstract}


Bias elimination and recent probing studies attempt to remove specific information from embedding spaces. 
Here it is important to remove as much of the target information as possible, while preserving any other information present. 
INLP is a popular recent method which removes specific information through iterative nullspace projections.
Multiple iterations, however, increase the risk that information other than the target is negatively affected.
We introduce two methods that find a single targeted projection: Mean Projection (MP, more efficient) and Tukey Median Projection (TMP, with theoretical guarantees). 
Our comparison between MP and INLP shows that (1) one MP projection removes linear separability based on the target and (2) MP has less impact on the overall space.
Further analysis shows that applying random projections after MP leads to the same overall effects on the embedding space as the multiple projections of INLP. 
Applying one targeted (MP) projection hence is methodologically cleaner than applying multiple (INLP) projections that introduce random effects. 





\end{abstract}

\section{Introduction}\label{sec:intro}

Word embedding spaces can contain rich information that is valuable for a wide range of NLP tasks.
High quality word embeddings should capture the semantic attributes associated with a word's meaning.
Though we can establish that points representing words with similar semantic attributes tend to be close to each other, it is challenging to reason which attributes are captured to what extent and, in particular, which dimensions capture these attributes \cite{sommerauer2018firearms}. 
But even though we do not (yet) know exactly which attributes are represented in which way, we nevertheless may want to ensure that one particular attribute is no longer present in the embedding. For example, because this attribute is a protected attribute such as gender which can lead to harmful bias, or because we want to test how the attribute was affecting a system in a probing setup.

There are various approaches to remove a particular attribute from an embedding.
Clearly it is important to remove as much of the target attribute as possible while preserving other information present. Arguably linear projections are one of the least intrusive methods to transform an embedding space. Iterative nullspace projection \cite[INLP]{nullitout} is a recent popular method which follows this paradigm and uses multiple linear projections. INLP achieves \emph{linear guarding} of the protected attribute, that is, a linear classifier can no longer separate instances that have the attribute from instances that do not. 

However, it remains unclear to what extent INLP affects the remainder of the embedding space. In general it requires 10-15 iterations (projections) to achieve linear guarding; it is hence likely that other attributes are negatively affected as well. This is particularly problematic in settings where INLP is used to determine what impact a specific attribute has on overall systems.

In this paper, we introduce two methods that find a single targeted projection which achieves linear guarding at once. Specifically, our \textbf{Mean Projection [MP]} method uses the mean of the points in each class to determine the projection direction. The mean is more robust to unbalanced class distributions than INLP. The mean can, however, be overly sensitive to outliers. We prove that our second method, the \textbf{Tukey Median Projection [TMP]}, finds nearly worst-case optimal projections for any input. 
Unfortunately, computing the Tukey median \cite{AlgCombGeom:87} is computationally expensive in high dimensions. In our experiments we hence compared only MP to INLP. Specifically, we carried out the gender debiasing experiments of \citet{nullitout} with both methods. They show that 
\begin{enumerate}
    \item MP only needs one projection for linearly guarding gender where INLP needs 10-15.
    \item MP has less impact on other aspects of the embedding space.
\end{enumerate}
INLP projections improve simlex scores and reduce WEAT scores more than MP. 
A priori it is unclear why their many projections should have this effect. 
We investigated and show that the same improvements appear after applying random projections (either after MP or after the first INLP projections). 
Applying one MP projection to linearly guard an attribute is hence methodologically cleaner than applying multiple (INLP) projections.


\section{Related Work}\label{sec:relatedwork}

Multiple methods for removing bias from embeddings have been suggested. Bias can be addressed at the level of the training data \citep{debias_dataset}, the training process \citep{adversarial_debias, fair_adversarial}, and the resulting model itself \cite{nullitout, bolukbasi, kernel}. Debiasing an existing model has the clear advantage that it can be done with relatively little data and that the model does not need to be retrained. In this paper, we focus on removing bias through transforming the embedding space. 

Transformations of embedding spaces can be performed by means of linear projections \cite{bolukbasi}, multiple linear projections \cite{nullitout}, or via non-linear kernels \cite{kernel}. \citet{nullitout} introduce a method called Iterative Nullspace Projection (INLP) where they attempt to remove the bias by projecting points iteratively. Our methods achieve linear guarding after a single linear projection. As such we directly improve upon the INLP method, which we describe in detail in Section~\ref{sec:problem}. In the remainder of this section, we focus on more recent post-training approaches and applications of INLP.

Similar to \citet{nullitout}, \citet{rlace-ravfogel22a} aim to find a linear subspace of the data such that removing that subspace using projection removes bias in the data optimally. Doing so necessarily requires a metric to measure how well bias has been removed by a specific projection. \citet{rlace-ravfogel22a} use the classification loss of various classifiers as metrics. That is, a higher classification loss after projection corresponds to better bias removal. This directly translates into a minimax optimization problem, which the authors refer to as a minimax game. The authors present different strategies to solve this optimization problem for their chosen classifiers. A similar strategy was previously described by \citet{obstruct_mfcs, obstruct}. Here the corresponding minimax optimization problem is referred to as ``maximizing inseparability''. \citet{obstruct} show that this problem can be solved efficiently under certain convexity assumptions on the metric.

\citet{zhang2022probing} independently developed a projection method which is very similar to our MP. Their method also uses class means to find the projection direction, but proceeds via a projection to the origin, and hence needs one additional projection. The goal of the paper, however, is quite different, namely the study of human brain reactions to syntactic and semantic features of words. As such, it does not include a systematic comparison to INLP or a comprehensive analysis of the impact on the embedding space.

\citet{gold_glitter} in recent and unpublished work compute linear projection vectors by minimizing the covariance between biased word vectors and the protected attribute. They generally need two projections to achieve comparable or better performance to INLP. The paper does not report on exactly the same experiments as we do, but for all experiments reported in both papers, our method MP performs similarly or better. 
  
\citet{dev-etal-2021-oscar} use a different kind of linear transformation on embedding spaces, namely rotation. Their goal is to disentangle two particular attributes, such as gender and occupation. To do so, they rotate the embedding space in such as way that the subspaces corresponding to the two attributes become orthogonal.
On the positive side, their approach does not remove a dimension and hence arguably retains more information in the embedding. However, it will not remove all bias with respect to gender, but only gender bias with respect to a single specified attribute, e.g., occupation. More generally, it will remove the interaction between the two specified attributes and not actually completely remove an attribute from the embedding.



INLP has recently gained importance for probing studies. Probing has been criticized because (1) results are difficult to interpret, and (2) it is usually not possible to test whether the target attribute is also relevant for downstream tasks. \newcite{elazar2021amnesic} propose \emph{amnesic probing}, which employs INLP to \emph{remove} the target attribute and then tests change in performance of the adapted embeddings on downstream tasks. INLP has since been used for this purpose in several other studies \cite[e.g.]{celikkanat-etal-2020-controlling,nikoulina2021rediscovery,babazhanova2021geometric,dankers2022can,gonen2022analyzing,lovering2022unit}. When using a projection-based removal method such as INLP or our MP in a probing setup, it is particularly important to ensure that all other information is preserved. 
If this is not the case, the probing study may lead to misleading conclusions. 
We hence recommend future such studies to consider using MP instead of INLP.


\section{Projection Methods}\label{sec:problem}

In this section we introduce our two new projection methods MP and TMP. Furthermore, we theoretically analyze their relative performance when (linearly) debiasing word embeddings, also in comparison to the existing method INLP. Here we focus on binary classification for ease of explanation. Both methods can however also be used for multi-class classification (see Section~\ref{ssec:original_ex}) and then require $n-1$ projections for $n$ classes.
We start with a formal description of the problem we study.

Our input is a set of word embedding vectors in $\Reals^d$. We interpret these vectors as a set of points $P = \{p_1, \ldots, p_n\}$ in $\Reals^d$ and often use the term ``point'' when referring to an embedding vector.
Every point (word) has an associated set of discrete attributes ($A = \{a_1, \dots, a_n\}$). 
Let $a^*$ be the attribute we aim to remove.
For simplicity, we assume that $a^*$ is a binary attribute with values $\{-1, +1\}$. We denote the two resulting classes of points as $P^- = \{p_i \in P \mid a^*_i = -1\}$ and $P^+ = \{p_i \in P \mid a^*_i = +1\}$.
We want to find a transformation $P'$ of our point set $P$ such that any linear classifier $C$ trained on $P'$ to classify $a^*$ cannot significantly outperform a classifier that labels by the majority class.
That is, $P'$ is \textbf{linearly guarded} with respect to $a^*$. 

In the following we consider only transformations which consist of one or more projections.
Let $p$ be a point in $P$. 
We define its projection $p_w$ along a unit vector $w$ as $p_w = p - (p \cdot w) w$. 
The projection along the vector $w$ maps points to the hyperplane $H_w$ which is orthogonal to $w$ ($H_w \bot w$) and contains the origin. 

To evaluate how well INLP, MP, and TMP do in terms of linear guarding, we consider the number of misclassifications with respect to $a^*$ by the best possible linear classifier after a single projection; the higher the number of misclassifications, the better the method performs.
If a single projection is sufficient to achieve linear guarding, then arguably the semantic encoding of other attributes in our word embeddings are preserved as well as possible.

From a theoretical perspective, TMP is the most effective method; it increases the number of misclassifications the most. However, it is costly to compute exactly for data in 300+ dimensions. 
MP is not as effective in theory as TMP, but it is more effective than INLP. 
Furthermore, MP is very easy to compute and appears to be very effective in practice. 
Hence, we evaluate the efficacy of MP extensively in Section~\ref{sec:experiments} and recommend to use MP for linear guarding in practice.

\begin{figure*}
\centering
\begin{minipage}[b]{\columnwidth}
\centering
    \includegraphics{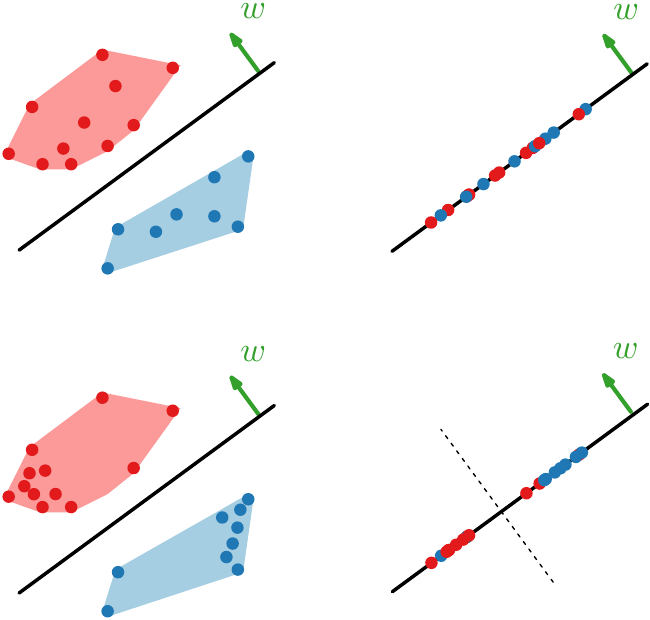}
    \caption{Attribute $a^*$ is represented by color; the resulting point sets after projection along $w$.}
    \label{fig:convexhull_problem}
\end{minipage}
\hfill
\begin{minipage}[b]{\columnwidth}
\centering
    \includegraphics{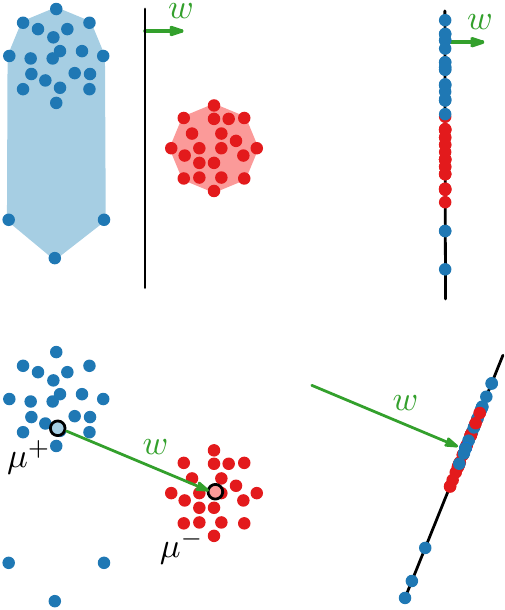}
    \caption{Attribute $a^*$ represented by color; results of INLP (top) and MP (bottom).}
    \label{fig:svm_vs_mean}
\end{minipage}
\end{figure*}

\subsection{INLP}
Iterative Nullspace Projection~\cite[INLP]{nullitout} debiases by iteratively projecting along a vector computed by a linear classifier. Here we assume that this linear classifier is a linear Support Vector Machine (SVM), as this is the main classifier used in their paper. In that setting, at each iteration, INLP trains an SVM to classify $a^*$ in the point set $P$.
Since $a^*$ is binary, the result of the SVM is a vector $w$ (along with a single real value, which we may ignore). 
The input points are then projected along $w$ and the resulting point set is used as the input for the next iteration.

If $P^-$ can be separated from $P^+$ with a hyperplane, then the vector $w$ is the normal of a separating hyperplane and indicates the direction along which there is the largest gap between $P^-$ and $P^+$: the SVM margin. If $P^-$ and $P^+$ cannot be separated by a hyperplane, then the classifier allows misclassifications for a certain penalty, but otherwise still maximizes the margin. In this setting, SVM uses a parameter that controls the trade-off between misclassifications and the margin.

\paragraph{Analysis.} The vector $w$ used in the projection essentially maximizes the gap between $P^-$ and $P^+$, and hence it is determined only by the points on the ``outside'' of $P^-$ and $P^+$. 
Specifically, if $P^-$ and $P^+$ can be separated, then the vector $w$ can be determined solely by the points on the convex hulls\footnote{The convex hull of a set of points is the smallest convex set that contains all of the points.} of $P^-$ and $P^+$, and is not influenced by the points that are interior to the convex hulls. 
If $P^-$ and $P^+$ are not separable, then $w$ will be determined by the points on the convex hulls of the correctly classified points. Thus, if the number of misclassifications is small, then many of the ``more interior'' points of $P^-$ and $P^+$ still do not contribute to determining the projection vector~$w$. 

\citet{obstruct} show that the projection along $w$ does eliminate the existence of a perfect linear classifier after projection (that is, there must be at least one misclassification), but this is by far not sufficient for linear guarding. We claim that we must also consider the points interior to the convex hulls of $P^-$ and $P^+$, if we wish to obtain an effective projection for linear guarding. 

We illustrate this claim with Figure~\ref{fig:convexhull_problem}, where $P^-$ and $P^+$ are colored red and blue, respectively. 
The black line indicates the hyperplane learned by a linear SVM and $w$ is the corresponding vector. 
The points at the top are distributed roughly evenly in the convex hull of their class (the colored polygon). Hence the projection along $w$ is very effective for linear guarding. 
The points at the bottom are distributed unevenly in the same convex hulls. 
The same projection along $w$ is now not very effective: red and blue points can still be separated fairly well by a linear classifier with relatively high accuracy.

We observe that the distribution of all points must be taken into account to compute a projection that is effective for linear guarding. 
Our two methods (MP and TMP) explore different ways of incorporating the distribution of the interior points when constructing the projection vector $w$.

\addtocounter{figure}{1}

\begin{figure*}[b]
    \centering
    \includegraphics{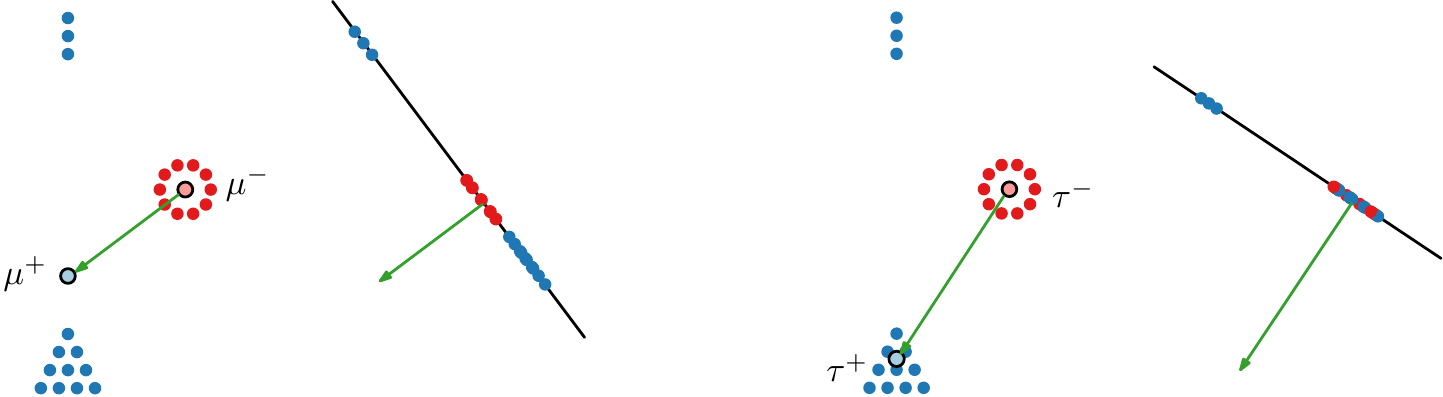}
    \caption{Results of MP (left) and TMP (right).}
    \label{fig:median_vs_mean}
\end{figure*}

\addtocounter{figure}{-2}

\subsection{Mean Projection (MP)}\label{subsec:MP}

Our general goal is to make the distribution of the points in $P^-$ and $P^+$ as similar as possible after the projection; that makes it difficult to distinguish the two sets using a linear classifier. 
One of the most characteristic values of a distribution is its \emph{mean}. 
Therefore, in the Mean Projection method we ensure that the means of $P^-$ and $P^+$ are identical after projection. 
Specifically, let $\mu^-$ be the mean of $P^-$ and let $\mu^+$ be the mean of $P^+$. 
We choose the projection vector simply as $w = \mu^+ - \mu^-$. 
Note that $\mu^+$ and $\mu^-$ (and hence $w$) can easily be computed efficiently by summing up the coordinates of the points in each point set and dividing by the number of points in that set. 

\paragraph{Analysis.} 
MP tends to be very effective, since the majority of the points are typically concentrated around their mean. 
Figure~\ref{fig:svm_vs_mean} compares INLP (top) and MP (bottom) in a scenario where MP clearly outperforms INLP with respect to linear guarding. 
However, we can also directly see a weakness of MP: the mean of a point set can be pulled into some direction by (few) outliers. 
Figure~\ref{fig:median_vs_mean} (left) shows another example of this phenomenon. 
Hence, there exist theoretical instances where MP performs quite poorly with respect to linear guarding.
However, we demonstrate in Section~\ref{sec:experiments} that MP is more effective than INLP in practice.

\subsection{Tukey Median Projection (TMP)}\label{subsec:tmp}
As illustrated above, the mean can be influenced by outliers. 
Ideally, we would hence like to identify a ``center point'' for each point set that has roughly the same number of points in every ``direction''. In 1D this center point is known as the \emph{median}. 
We therefore consider a higher-dimensional version of the median as our center point.     
    
Let $P$ be a point set in $\Reals^d$ and let $q \in \Reals^d$ be another point which is not necessarily in $P$. The \emph{Tukey depth} $t(q)$ of $q$ with respect to $P$ is the smallest number of points in $P$ in any closed halfspace bounded by a hyperplane in $\Reals^d$ that passes through~$q$. In Figure~\ref{fig:tukey_depth}, $q_1$ is a point with Tukey depth $3$ and $q_2$ is a point with Tukey depth $5$. A \emph{Tukey median} $\tau$ of a point set $P$ is a point $q \in \Reals^d$ with the maximum Tukey depth with respect to $P$ among all points in $\Reals^d$. Point $q_2$ in  Figure~\ref{fig:tukey_depth} is a Tukey median of the black points. 

Now let $\tau^-$ be a Tukey median of $P^-$ and let $\tau^+$ be a Tukey median of $P^+$. Then the projection vector of the Tukey Median Projection method is simply $w = \tau^+ - \tau^-$.

\begin{figure}[t]
    \centering
    \includegraphics{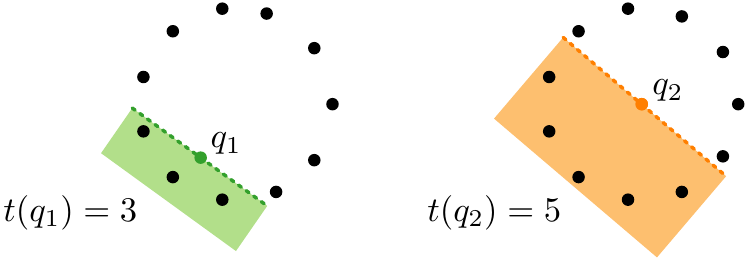}
    \caption{Two points and their Tukey depth.}
    \label{fig:tukey_depth}
\end{figure}

\paragraph{Analysis.} Figure~\ref{fig:median_vs_mean} contrasts MP and TMP. 
The mean of the blue point set is clearly not in the center of the majority of the points due to the outliers. As a result, the projection of MP is not very effective for linear guarding. The Tukey median of the blue point set, however, remains centered near the majority of the points and hence TMP is effective. 

One may wonder if there is also an example where TMP performs poorly. We can in fact prove that TMP in general performs quite well for linear guarding. Specifically, we can show that the number of misclassifications after TMP is at least the minimum of the Tukey depths of $\tau^+$ and $\tau^-$. It is known that, for every point set $P$ in $\Reals^d$ with $n$ points, the Tukey median with respect to $P$ has Tukey depth at least $\left\lceil\frac{n}{d+1}\right\rceil$ (see, e.g.,~\newcite{AlgCombGeom:87}). Finally, we can show that there exist point sets $P^+$ and $P^-$ such that, after any projection along a single vector, the number of misclassifications is at most $O(n/d)$. Thus, TMP is asymptotically worst-case optimal for maximizing the number of misclassifications after projection.
The exact theorem statements and their proofs can be found in Appendix~\ref{app:tukey_proof}.

Unfortunately, TMP has one major drawback: computing the Tukey median of a set of points is computationally expensive in high dimensions. Although it is possible to compute an approximate Tukey median efficiently, the Tukey depth of this approximation may be significantly worse than that of the Tukey median (namely, $O(n/d^2)$)\footnote{We are aiming to maximize the number of misclassifications after projection, so $O(n/d)$ is better than $O(n/d^2)$}. Hence, we generally recommend to use MP in practice.

\section{Experiments}\label{sec:experiments}

 In this section, we present our main experimental results.\footnote{The code is available at \href{https://github.com/tue-alga/debias-mean-projection}{tue-alga/debias-mean-projection}.} We first describe the outcome of the original INLP experiments and how our results compare (Section~\ref{ssec:original_ex}). We then outline the additional experiments we carried out to gain more insight into surprising results of the original INLP experiments as well as those of our own method (Section~\ref{ssec:divingdeeper}).
 
\begin{table*}[t]
    \centering
    \begin{tabular}{l|c|c|l}
        \textbf{Experiment} & \textbf{INLP} & \textbf{MP} & \textbf{Interpretation of results}\\ \hline
        Linear Guarding acc. & 34.9 \% & 34.14 \%  & Similar\\
        non-linear Classification acc. & 85.0\% & 81.6\% & MP slightly better\\
        Effect on Word Neighborhoods & - & - & MP better (more stable neighborhoods)\\
        Embedding Space simlex-999 corr. & 0.489 & 0.373 & INLP better (*)\\ \hline \hline
        WEAT 6 & 0.008 & 0.173 & INLP better (*)\\
        WEAT 7 & 0.084 & 0.191 & INLP better (*)\\
        WEAT 8 & -0.310 & 0.110 & MP better\\ \hline \hline
        Bias-by-neighbor & 73.4\% & 74.5\% & Similar\\ \hline \hline
        Gender in the wild BOW acc. & 77.1\% & 76.7\% & Similar\\
        Gender in the wild FastText acc. & 73.6\% & 75.6\% & Similar\\
        Gender in the wild BERT acc. & 74.7\% & 75.2\% & Similar\\
        Gender in the wild BOW TPR-GAP & 0.111 & 0.001 & MP better\\
        Gender in the wild FastText TPR-GAP & 0.103 & 0.092 & Similar\\
        Gender in the wild BERT TPR-GAP & 0.065 & 0.083\% & Similar\\
    \end{tabular}
    \caption{Summary of the results in Section~\ref{ssec:original_ex}: MP generally performs equal to or better than INLP. (*) In Section~\ref{ssec:divingdeeper} we show that adding random projections to~MP achieves the same performance as INLP.}
    \label{tab:results}
\end{table*}

\subsection{Original Experiments}\label{ssec:original_ex}

We compare the performance of INLP and MP on the gender bias experiments reported in Sections 6.1 and 6.3 in \newcite{nullitout}. We carried out all experiments with the original (INLP) and our new (MP) method. We could successfully reproduce Ravfogel et al.'s \citeyear{nullitout} results. Since the results were nearly identical, we report the INLP results from the original paper. For reasons of space, this section only contains the intuition and main results. We provide a self-contained description in Appendix~\ref{app:detailed_experiments}, where all experimental details can be found. We use the same settings and metrics as \citet{nullitout}, unless specified otherwise. Table~\ref{tab:results} summarises the results in this section.

\paragraph{General Settings.} 

We use the GloVe word embeddings~\cite{zhao-glove}\footnote{\texttt{glove.42B.300d}} and limit the dataset to the 150,000 most common words. We create the same classes of male, female and neutral embeddings as \newcite{nullitout} following their approach.\footnote{Gender is not binary. Our experiments are restricted to removing male and female gender bias because we follow the settings from \newcite{nullitout}. Our limitations and ethics sections provide more elaborate comments on this matter.} A male vector $\vec{M}$ is defined as $\vec{he} - \vec{she}$ and a female vector $\vec{F}$ as $\vec{she} - \vec{he}$. Our male dataset consists of the 7500 data points closest to $\vec{M}$ and our female dataset to the 7500 data points closest to $\vec{F}$. We use the same random selection of 7500 neutral words with a cosine similarity of less than 0.3 to $\vec{M}$ as \newcite{nullitout}. 

We include experiments on two classes (\emph{feminine} and \emph{masculine}) and on three classes (\emph{feminine}, \emph{masculine} and \emph{neutral}). The INLP method uses the datasets to train SVM classifiers as described above. For the MP method, let $\vec{v_F}$, $\vec{v_M}$ and $\vec{v_N}$ be the mean of respectively feminine, masculine and neutral labeled points of a given set $D$. Then in experiments on two classes the mean projection method uses the $\vec{v_F} - \vec{v_M}$ vector for projection. For experiments with all three classes, we use both $\vec{v_N} - \vec{v_F}$ and $\vec{v_N} - \vec{v_M}$ vectors for projection.

\paragraph{Linear Guarding.} We first compare the process of obtaining linear guarding using INLP and MP. Figure~\ref{fig:acc_compare_three_classes} illustrates the decrease in accuracy of a linear classifier trained to identify gender. The Mean Projection method brings down the accuracy to 34.18\%. It takes INLP 12 iterations to reach a comparable result of 34.9\% or at least 6 iterations to reach 39.9\% accuracy. Figure~\ref{fig:acc_compare_two_classes_app} in Appendix~\ref{app:detailed_experiments} shows a similar pattern for binary classification, where MP reaches 50.6\% and INLP needs 14 iterations to drop to 50.51\%.

\begin{figure}
    \centering
    \includegraphics{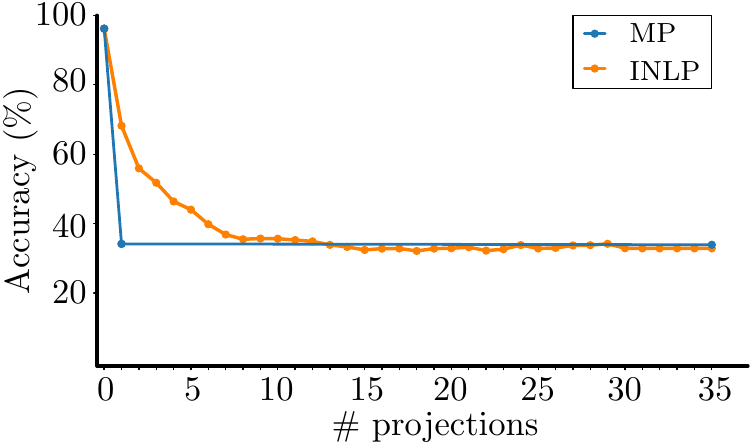}
    \caption{Comparing accuracy of INLP and MP on \textbf{three} classes of feminine, masculine and neutral words.}
    \label{fig:acc_compare_three_classes}
\end{figure}

\paragraph{Classification.} Both methods aim to remove linearly encoded information. \newcite{nullitout} test to what extent the non-linear encoding of gender remains intact by running a 1-hidden-layered MLP with ReLU activation and report 85.0\% after 35 iterations with INLP. After a single MP projection, accuracy of the MLP drops to 81.6\%.

\paragraph{Effect on the Embedding Space.} \newcite{nullitout} show changes in the three nearest neighbors of 40 randomly selected words and 22 gendered given names (see Tables~\ref{tab:emb_space_random} and~\ref{tab:emb_space_gendered_name} in Appendix~\ref{sec:embspacechange}). MP keeps these environments more stable changing 8 out of 120 neighbors compared to 43 for INLP for random words. MP changes 24 and INLP 56 tokens out of 66 neighbors of the given names. We also observe that 14 of the 24 changes are other given names (5 of which of opposite gender). With INLP, 46 new tokens are not regularly spelled names.

\newcite{nullitout} also investigate how INLP impacts semantic similarity scores on multiple datasets reporting an increase of results, e.g.\ from 0.373 to 0.489 on simlex-999. Applying MP has less impact with a result of 0.385. This confirms that MP has less impact on the overall space, but also raises the question whether the additional changes caused by INLP's extra projections is a positive effect. We investigate this in Section~\ref{ssec:divingdeeper}.

\paragraph{WEAT.} \citet{weat} evaluate bias in embeddings by comparing the distance of known stereotypical terms to terms that are either explicitly male or explicitly female. A higher WEAT score means that stereotypical terms are indeed closer to their stereotyped gender. Like \newcite{nullitout} we use gendered names for representing attributes and take the targets of \emph{WEAT 6},  \emph{7}, and \emph{8}. MP achieves WEAT scores of 0.173, 0.191 and 0.110 respectively. The first INLP iteration yields scores of 0.278, 0.347, 0.203 and the full 35 iterations reach 0.008, 0.084 and -0.310 respectively. We investigate this in Section~\ref{ssec:divingdeeper}.

\paragraph{Bias-by-neighbor.} \newcite{lipstick} propose an evaluation that checks whether the 100 nearest neighbors of a term after debiasing carried the same gender bias before debiasing. \newcite{nullitout} show that INLP reduces the gender bias of Bolukbasi et al.'s \citeyear{bolukbasi} set of stereotypical professions from 85.2\% to 73.4\%. MP reaches a comparable reduction (74.5\%).

\paragraph{TPR-GAP: ``gender in the wild''.} \citet{biasbios} propose an approach that measures to what extent debiasing methods can remove gender bias from a system that predicts occupations based on biographies. A fair system should have equal performance for members of different classes and not amplify a bias that is present in the label distribution \cite{hardt2016equality}.  In the case of gender, the system should perform equally well predicting occupations such as surgeon, caretaker, secretary or marine, regardless of whether the biographical texts are about men or women. \citet{biasbios} use the $GAP^{TPR}$ score to measure this. The GAP$_{fem}$, e.g., quantifies to what extent a given occupation is more probable to be predicted correctly for biographies of women compared to biographies of men. A good debiasing approach should obtain a $GAP^{TPR}$ score that is close to zero while maintaining the classification accuracy that was obtained before debiasing.

Instead of directly debiasing embeddings, we debiase representations of entire biographies within the occupation classification system. Like~\citet{nullitout}, we use logistic classifiers that take one of three representations of the biographies as input: (1) one-hot BOW, (2) averaged FastText embeddings \citep{fasttext} and (3) the last hidden state of BERT \cite{devlin2019bert}. The projections are applied using \newcite{pedregosa2011scikit}.

In both approaches, biography representations are debiased on the basis of their gender nullspace for each of the 28 occupations. For the MP setup, we create a mean projection for every occupation $o$ by identifying the projection vector $P_o$ based on the mean of all female and the mean of all male biographies with occupation $o$.  The INLP setup uses logistic regression in 100 iterations for the BOW representations, 150 linear SVM iterations for FastText representations and 300 linear SVM iterations of for the BERT.

Overall, results of both methods are comparable. INLP maintains a higher accuracy on identifying professions for one-hot BOW (77.1\% vs 76.7\%), where MP has higher accuracy for FastText (75.6\% vs 73.6\%) and BERT (75.2\% vs. 74.7\%). For BOW, the GAP score drops by 99.4\% when using MP compared to 35\% for INLP. Differences for other representations are smaller: MP reduces GAP by 49.5\% on FastText compared to INLP's 43.4\%. For the BERT model, INLP outperforms MP reducing the GAP by 64.3\% compared to MP's 52.2\%. See Table~\ref{tab:biography_results} in Appendix~\ref{app:detailed_experiments} for the full results.

\paragraph{Summary.} In general, we observe that a single MP projection achieves linear guarding where multiple INLP projections are needed and that the rest of the space remains more stable. The improvement of similarity results, as well as INLP's results on WEAT, raise further questions. \newcite{nullitout} hypothesize that the improvied similarity results may be due to a significant gender component in embeddings that does not correlate with human similarity judgment \cite[p.\ 7253]{nullitout}. Together with the increased drop in WEAT, this may point to 35 INLP iterations removing more gender information than the single MP projection. This is countered by the result of the non-linear classifier where MP induced a larger decrease than INLP. We investigate the results on similarity and WEAT further in the next subsection.


\subsection{Diving Deeper}\label{ssec:divingdeeper}

\paragraph{Simlex-999.} We first dive into the impact of INLP on simlex-999 \cite{hill2015simlex}. Figure~\ref{fig:simlex_compare} illustrates the changes in similarity scores (Correlation) after every iteration of INLP compared to MP (one projection). We observe that the scores mainly increase between the 8th and 14th iteration, remaining relatively stable before and afterwards. Going back to Figure~\ref{fig:acc_compare_three_classes}, we observe that this increase thus starts when the INLP projections have almost dropped to majority class accuracy, i.e.\ linear guarding. The main increase in similarity scores thus occurs \textbf{after} gender encoding has been removed, which could imply that the increased result is not related to gender. We investigate this by comparing the impact of INLP iterations to iterations that project the dataset along random vectors.\footnote{Weighted by eigenvalues of all principle components i.e. we choose a random direction from the span of the data.}

\begin{figure}[b]
    \centering
    \includegraphics{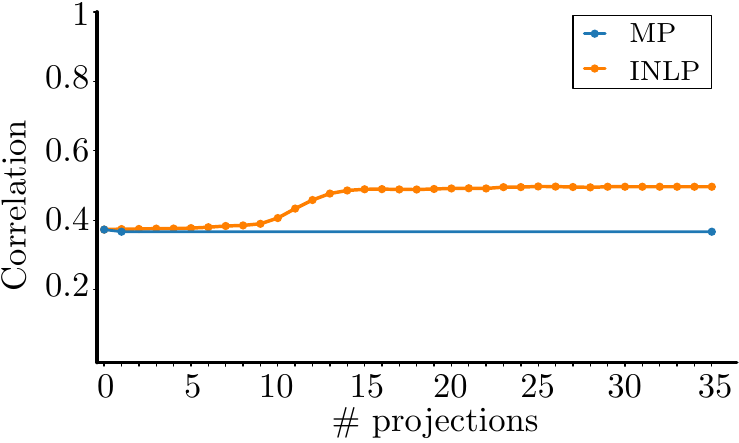}
    \caption{Change in simlex-999 scores for MP and INLP after every iteration.}
    \label{fig:simlex_compare}
\end{figure}

We compare three scenarios: directly applying 35 iterations with random projections to the original model [Random], adding 34 random projection iterations after one MP projection [MP+R] and adding 27 random projection iterations to 8 INLP projections (the point where the score starts to increase) [INLP-8+R]. We carry out 500 runs of this experiment and report the 95\% confidence interval of similarity correlation scores for each setting in Table~\ref{tab:correlations}. Running 35 iterations of INLP increased the semantic similarity score from 0.373 to 0.489. This improvement falls within the confidence interval of MP+R and INLP-8+R. Results of random projections only [Random] do not reach this score, but still clearly improve compared to the original 0.373 with a mean score of 0.447. We thus conclude that the improvement in similarity scores are due to reducing dimensions in general rather than removing (partial) representations of gender.

\begin{table}[t]
    \centering
    \begin{tabu} {X[1.2,l]  X[1.2 c] X[0.8 c] X[0.9 c]} 
                 &  $95\%$ CI & Mean & Std.  \\
         INLP-8+R  & $[0.47, 0.50]$    &  $0.486$  &  $0.009$ \\
         MP+R     & $[0.46, 0.50]$    &  $0.478$  &  $0.009$   \\
         Random   & $[0.43, 0.46]$    &  $0.447$  &  $0.009$
    \end{tabu}
    \caption{Simlex-999 scores for random projections.}
    \label{tab:correlations}
\end{table}

\paragraph{WEAT.} We investigate the WEAT scores in a similar manner. We first inspect the impact on the WEAT score per INLP iteration and then compare this to the impact of projections along weighted random vectors. Figure~\ref{fig:weat_random} illustrates the impact of applying INLP for 35 iterations (blue line) and 34 random projection iterations applied after applying MP method (500 runs). Recall that an ideal WEAT score would be close to zero. We observe that MP+R on average ends up with an equal score for \emph{WEAT 6}, a slightly higher positive (thus worse) score for \emph{WEAT 7} and close to zero for \emph{WEAT 8} where INLP results in a negative score. 

\begin{figure*}[h]
    \centering
    \includegraphics{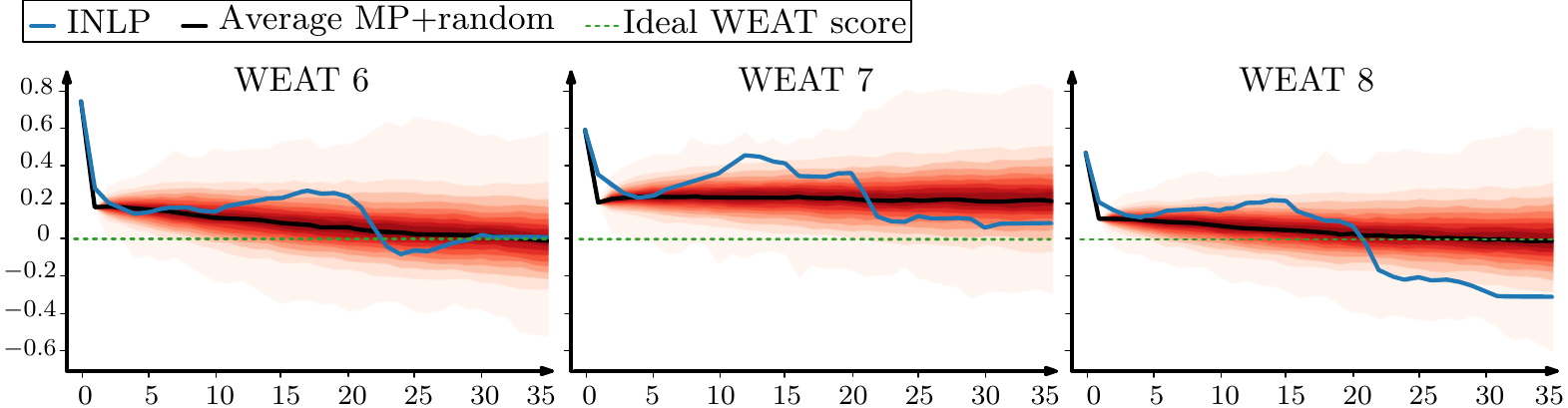}
    \caption{Plot of 500 runs where we first apply MP and then continue with 34 iterations of random projections. The shades of red indicate various levels of confidence (10 levels) in WEAT scores ($y$-axis) after a certain number of iterations ($x$-axis): darker red shades reflect WEAT scores occurring with a higher probability.} 
    \label{fig:weat_random}
\end{figure*}

We also compare the effect of applying random projections after 8 INLP iterations. INLP-8+R shows similar results as MP+R. Applying random projections from the start [Random] leads to slightly increasing WEAT scores in most cases (see Figures~\ref{fig:inlp_weat_random} and~\ref{fig:weat_random_app} in Appendix~\ref{app:diving_even_deeper}). This implies necessity of the  first step in removing gender to make a reduction of WEAT by random projections. 

We now turn back to what this means for iterative nullspace projections. When inspecting the blue lines in Figure~\ref{fig:weat_random}, we observe that all WEAT scores initially drop to the same level as with MP after 2-4 iterations and then increase again. Between 20 and 25 iterations we observe a sharp drop, even flipping to a bias in the opposite direction for \emph{WEAT 8}. These patterns do not reveal a consistent positive impact of applying additional nullspace projections. 
\citet{ethayarajh-etal-2019-understanding} report that results obtained by applying WEAT are brittle and that statistically significant results in opposite directions can be obtained depending on the attribute words selected. They show this in a highly simplified setting (using only one attribute word at the time). Nevertheless, the WEAT sets are relatively small (8 attributes and targets per set) and some of the results we see could be the effect of the specific selection of terms in the set. Overall, this makes the results even more difficult to interpret. We can say, however, that a single MP and the first INLP projections seem beneficial for decreasing WEAT and there is no evidence that continuing iterations of INLP further removes gender bias.

\paragraph{Summary.} Our further investigations showed that (1) increases in simlex-999 scores occur very locally after gender has been removed and that (2) both for similarity and WEAT, random projections after MP or a limited number of INLP iterations lead to similar results as 35 INLP iterations. We therefore conclude that the side effects of continuing to apply INLP are not related to the overall goal of removing encoding of gender, but rather related to the overall effect of dimension reduction.


\section{Conclusion}\label{sec:conclusion}

This paper started from the idea that one targeted projection results in similar linear guarding and more stability of the remaining space compared to iterative nullspace projections. We proposed two methods to find such projections: Mean Projection (MP) and Tukey Median Projection (TMP). We compared performance of Mean Projection, which is still susceptible to outliers but more efficient to compute than TMP, to the gender based experiments using 35 Iterative Nullspace Projections (INLP) reported in~\newcite{nullitout}. 

Our results show that MP obtains comparable linear guarding in one projection where INLP requires multiple iterations. They also confirm that the rest of the space remains more stable through fewer changes in nearest neighbors of randomly chosen words and similarity scores that are closer to the original scores compared to INLP. Two results led to further questions: similarity scores improved after applying INLP and INLP resulted in better \emph{WEAT} scores than MP. We conducted additional experiments to test whether the extra debiasing iterations performed in INLP end up removing more subtle representations of gender bias that are missed by the single MP iteration. 


The results show that it is unlikely that INLP can indeed remove more subtle representations of bias for two reasons: (1) we observed that most effects occur once linear guarding has already been achieved. (2) More importantly, similar performance is obtained by running random projections after either MP or the first 8 INLP iterations. 
We therefore conclude that these additional effects are not related to removing representations of the targeted attribute, but rather a side effect of reducing the dimensionality of the space. This leads us to the overall conclusion that a single targeted projection is indeed preferable over iterative nullspace projections when aiming to remove an attribute. This finding is particular important for the recent line of research that uses INLP for interpretability: these studies test the effect of removing a target attribute and their conclusions should not be confounded by other changes to the space. 

Our theoretical findings in Section~\ref{sec:problem} suggest that MP might produce poor results for skewed data. Hence we plan to analyze MP more rigorously on synthetic data, specifically, data where the points are not centered around the mean. In this way we hope to identify scenarios where MP produces poor projections; subsequently we can test if TMP produces better projections for the same scenarios.

\section{Limitations}

The methods we described and evaluated in this paper have the following limitations.

\paragraph{Dependence on labelled attribute data.} The results of all projection methods (INLP, MP and TMP) directly depend on the representativeness of the labelled data points that are used to identify the projections. All methods would suffer from non-representative points. In addition, outliers can hamper the effectiveness of MP. 

\paragraph{Binary gender.} The binary perspective on gender used in our experiments does not  reflect the  reality of non-binary gender identities. These experiments focus on gender bias in the form of stereotypical associations with binary gender identities. The underrepresentation of non-binary gender references (e.g.\ singular \textit{they}, \textit{them} in English) is another form of bias that also affects NLP systems. Recent work has shown that current NLP technologies indeed have difficulties dealing with non-binary gender references \cite[e.g.]{dev2021harms,brandl-etal-2022-conservative}. To our knowledge, the question of whether language models also contain stereotypes around non-binary gender has not been addressed yet. If such bias is indeed present, data sparseness is likely to pose a problem for the methods described in this paper.

\paragraph{Suitability for language with grammatical gender is unclear} For the application of gender bias, we note that all experiments are run on English. In particular for languages with grammatical gender, it is unclear whether these projections would be able to identify semantic gender bias correctly while avoiding to remove grammatical characteristics of the language: words with the same semantic gender will almost always also have the same grammatical gender. Projections are likely to remove this highly distinctive feature which is also a grammatical property. Its removal may impact downstream applications.

\paragraph{Linear guarding only.} The methods described here aim only at (and achieve only) linear guarding. This means that non-linear representations of the attributes remain present in the data (e.g.\ the one-hidden layered perceptron still achieves an accuracy of 85.0\% after 35 INLP iterations and 81.6\% after applying MP on binary classification).

\section{Ethics}

We applied our method to gender debiasing. The approach aims to remove bias which can be seen as helping to address an ethical issue, rather than causing one. Nevertheless, we feel it is important to point out the following. When using the methods discussed in this paper to remove a protected attribute, with the idea to avoid harmful bias, the limitations outlined above \textbf{must} be taken into account. This concerns both potential limitations of the effectiveness of the approach and the treatment of gender as a binary variable. Concerning effectiveness, this means considering the fact that (1) non-linear representations are not addressed by these methods and (2) the success of the approach is directly dependent on the quality of data used to identify the projections.  Concerning the treatment of non-binary gender, it should be kept in mind that (1) to our knowledge, not much is known yet about stereotypical representations of non-binary gender in language models (2) it is known that current NLP methods experience difficulties in dealing with non-binary gender. In general, debiasing methods must always be carefully evaluated to avoid that users of a system assume the problem is solved, when this is not (completely) the case. Ideally, such evaluation should look beyond the scope of the target. When dealing with gender, this should include looking at what models are doing with non-binary gender references.

In practice, the considerations outlined above mean the following: We do not claim that any of the approaches presented in this paper can fully remove bias from embedding representations. When using debiased embeddings as part of a system, we strongly advise to conduct a critical evaluation with respect to potentially remaining bias using data that are representative of the use-case. This is particularly important if the use-case encompasses the potential for discrimination (e.g.\ automatic CV analysis). 



\bibliography{anthology,custom}
\bibliographystyle{acl_natbib}

\cleardoublepage
\appendix

\section{Tukey Median Projection}
\label{app:tukey_proof}

\begin{theorem}\label{thm:tukeybound}
Let $q^-$ be a point with Tukey depth $t(q^-)$ with respect to $P^-$ and let $q^+$ be a point with Tukey depth $t(q^+)$ with respect to $P^+$. Then the projection of $P$ along vector $w = q^+ - q^-$ ensures that any linear classifier must make at least $\min(t(q^-), t(q^+))$ misclassifications after projection.  
\end{theorem}
\begin{proof}
Let $\tau$ be the Tukey medians of $P^-$ and $P^+$ after projection (by construction, these Tukey medians have been projected onto the same point). Now let $H$ be the hyperplane corresponding to some linear classifier after projection. Let $H_\leftrightarrow$ be the hyperplane obtained by shifting $H$ to contain the point $\tau$. Let $H^-_\leftrightarrow$ be the halfspace bounded by $H_\leftrightarrow$ that contains the majority of $P^-$ after projection, and let $H^+_\leftrightarrow$ be the other halfspace bounded by $H_\leftrightarrow$. By construction of $\tau$, there must be at least $t(q^-)$ points of $P^-$ in $H^+_\leftrightarrow$ and at least $t(q^+)$ points of $P^+$ in $H^-_\leftrightarrow$, which are misclassified by the linear classifier defined by $H_\leftrightarrow$. By shifting the hyperplane back from $H_\leftrightarrow$ to $H$, we can see that either the misclassified points of $P^-$ or the misclassified points of $P^+$ remain misclassified by the linear classifier defined by $H$. Thus, any linear classifier must misclassify at least $\min(t(q^-), t(q^+))$ points after projection by TMP.
\end{proof}

We can show that TMP is asymptotically worst-case optimal in minimizing the number of misclassifications after projection. For that we need the following technical lemma.

\begin{lemma}\label{lem:simplex}
Let $\Delta_d$ be the regular $d$-dimensional simplex embedded in $\mathbb{R}^d$ such that the center of $\Delta_d$ lies at the origin and the $d+1$ vertices of $\Delta_d$, represented by the vectors $v_1, \ldots, v_{d+1}$, are at unit distance from the origin. For any point $p \in \mathbb{R}^d$, there exist a unit vector $r$ such that $(v_i \cdot r) \leq (p \cdot r) - \frac{1}{d}$ for at least $d$ vertices of $\Delta_d$. 
\end{lemma}
\begin{proof}
As the vectors $v_1, \ldots, v_{d+1}$ span the entirety of $\mathbb{R}^d$, we can write $p$ as an affine combination $p = \sum_{i=1}^{d+1} \alpha_i v_i$, where $\sum_{i=1}^{d+1} \alpha_i = 1$. Without loss of generality, let $\alpha_j$ be the largest coefficient among all $\alpha_i$ for $1 \leq i \leq d+1$. We pick $r = v_j$. Since $\Delta_d$ is a regular simplex, it is known that $(v_i \cdot v_j) = \frac{-1}{d}$ for $v_i \neq v_j$. Using the linearity of the dot product, we now get the following:
\begin{align*}
(p \cdot r) &= \sum_{i=1}^{d+1} \alpha_i (v_i \cdot r) \\
&= \alpha_j - \frac{1}{d} \sum_{1 \leq i \leq d+1, i \neq j} \alpha_i \\
&\geq \alpha_j - \frac{1}{d} \sum_{1 \leq i \leq d+1, i \neq j} \alpha_j \\
&= \alpha_j (1 - \frac{d}{d}) \\
&= 0.
\end{align*}
On the other hand we have that $(v_i \cdot r) = \frac{-1}{d}$ for all $v_i \neq v_j$, and hence the stated inequality holds for all $v_i \neq v_j$. 
\end{proof}

We can now prove our main result. 

\begin{theorem}\label{thm:projecthard}
For any $d > 0$ and $n \geq m > 0$, there exist point sets $P$ of $m$ points and $Q$ of $n$ points in $\mathbb{R}^{d+1}$, such that for any projection of $P$ and $Q$ along a unit vector $w$, there exists a hyperplane $H$ in $\mathbb{R}^{d+1}$ with at least $\left\lfloor\frac{m d}{d+1}\right\rfloor$ points of $P$ on one side of $H$ and all points of $Q$ on the other side of $H$ after projection.
\end{theorem}
\begin{proof}
Let $e_1, \ldots, e_{d+1}$ be the basis vectors spanning $\mathbb{R}^{d+1}$. We first prove the statement for the case that $m = d+1$. Specifically, let $P$ be formed by the vertices of the regular $d$-dimensional simplex (as in Lemma~\ref{lem:simplex}) centered at the origin and restricted to the $d$-dimensional subspace spanned by the first $d$ basis vectors $e_1, \ldots, e_d$. Furthermore, let $Q$ be an arbitrary set of $n$ points such that all points are within Euclidean distance less than $\varepsilon$ of the center point $q = (0, \ldots, C)$ (for some $C > 0$ and $\varepsilon > 0$) in $\mathbb{R}^{d+1}$.

Let $w$ be any unit projection vector. We write $w$ as $w = \alpha e_{d+1} + \beta z$, where $z$ is a unit vector in the subspace spanned by $e_1, \ldots, e_d$ (such that we have $\alpha^2 + \beta^2 = 1$). Let $q_w$ be the center of $Q$ after projection along $w$. It is easy to see that $P$ and $Q$ are linearly separable after projection if $\|q_w\| \geq 1+\varepsilon$ (the center of $P$ remains at the origin after projection). We have that:
\begin{align*}
q_w &= q - (w \cdot q) w \\
&= C e_{d+1} - \alpha C (\alpha e_{d+1} + \beta z) \\
&= e_{d+1} C (1 - \alpha^2) - \alpha \beta C z.
\end{align*}
Since $z$ is independent from $e_{d+1}$ we get that $\|q_w\| \geq C (1 - \alpha^2) = C \beta^2$, and hence $\|q_w\| \geq 1+\varepsilon$ if $\beta^2 \geq \frac{1+\varepsilon}{C}$. We may therefore assume that $\beta^2 < \frac{1+\varepsilon}{C}$. 

Now let $x$ be the $d$-dimensional point obtained by omitting the $(d+1)^{st}$ coordinate of $q_w$. We apply Lemma~\ref{lem:simplex} to $P$ with $p = x$ to obtain a corresponding unit vector $r$. We extend $r$ with an additional coordinate with the value $0$ to obtain the $(d+1)$-dimensional vector $R = (r, 0)$. We then define the hyperplane $H$ as the hyperplane that satisfies $(p \cdot R) = (q_w \cdot R) - \varepsilon$ for points $p \in \mathbb{R}^{d+1}$.

We now show that $H$ has the desired properties. Let $u_i$ be one of the $d$ points of $P$ such that $(u_i \cdot r) \leq (x \cdot r) - \frac{1}{d}$, and let $u'_i$ be the point obtained by projecting $u_i$ along $w$. We get that:
\begin{align*}
u'_i &= u_i - (w \cdot u_i) w \\
&= u_i - \beta (z \cdot u_i) w \\
&= u_i - \beta'_i w,
\end{align*}
where $\beta'_i \in [-\beta, \beta]$. By construction of $x$, we also get that $(x \cdot r) = (q_w \cdot R)$. We then get the following:
\begin{align*}
(u'_i \cdot R) &= (u_i \cdot R) - \beta'_i (w \cdot R) \\
&= (u_i \cdot r) - \beta'_i \beta (z \cdot R) \\
&\leq (u_i \cdot r) + \beta^2 \\
&\leq (x \cdot r) + \beta^2 - \frac{1}{d} \\
&= (q_w \cdot R) + \beta^2 - \frac{1}{d}.
\end{align*}
Thus, if $\beta^2 - \frac{1}{d} < -\varepsilon$, then $(u'_i \cdot R) < (q_w \cdot R) - \varepsilon$. Since $\beta^2 < \frac{1+\varepsilon}{C}$, we can choose $C = 4d$ and $\varepsilon = \frac{1}{2d}$ to ensure this property, as then $\frac{1+\varepsilon}{C} - \frac{1}{d} < \frac{1}{2d} - \frac{1}{d} = -\varepsilon$. 

Now let $v_j$ be one of the points of $Q$, and let $v'_j$ be the point obtained by projecting $v_j$ along $w$. Since the projection along $w$ can only shrink distances, we get that $\|v'_j - q_w\| \leq \|v_j - q\| < \varepsilon$. This implies that $((v'_j - q_w) \cdot R) > -\varepsilon$, which can be rewritten as $(v'_j \cdot R) > (q_w \cdot R) - \varepsilon$. As a result, $d$ points of $P$ are on one side of $H$ and $n$ points of $Q$ are on the other side of $H$, as required.

If $P$ contains $m \neq d+1$ points, then we can simply follow the construction above and distribute the points arbitrarily close to the vertices of the regular simplex, such that there are at at most $\left\lceil\frac{m}{d+1}\right\rceil$ points around each vertex. As a result, $n$ points of $Q$ will be on one side of $H$ and at least $m - \left\lceil\frac{m}{d+1}\right\rceil = \left\lfloor\frac{m d}{d+1}\right\rfloor$ points of $P$ will be on the other side of $H$, which completes the proof.
\end{proof}

\begin{cor}\label{cor:misclassify-upper}
For any $d > 1$ and $n \geq m > 0$, there exists a set of points $P$ in $\mathbb{R}^{d}$ with $|P^-| = m$ and $|P^+| = n$ such that, for any projection of $P$ along a single unit vector $w$, the number of misclassifications of the best possible linear classifier after projection is at most $\left\lceil\frac{m}{d}\right\rceil$.  
\end{cor}
\begin{proof}
Choose $P$ and the binary attribute $a^*$ such that $P^-$ and $P^+$ match $P$ and $Q$ in the statement of Theorem~\ref{thm:projecthard} (note that the theorem has to be applied with dimension $d-1$). Then, for any projection of $P$ along a single vector, there exists a hyperplane $H$ such that $\left\lfloor\frac{m (d-1)}{d}\right\rfloor$ points of $P^-$ are on one side of $H$ and all points of $P^+$ are on the other side of $H$. Thus, the linear classifier defined by $H$ misclassifies at most $\left\lceil\frac{m}{d}\right\rceil$ points, and hence the best possible linear classifier must perform at least as well.   
\end{proof}

We now return to the number of misclassifications caused by TMP. It is known that, for every point set $P$ in $\Reals^d$ with $n$ points, the Tukey median with respect to $P$ has Tukey depth at least $\left\lceil\frac{n}{d+1}\right\rceil$. As a result, by using the projection of TMP, any linear classifier must make at least $\left\lceil\frac{m}{d+1}\right\rceil$ misclassifications after projection by Theorem~\ref{thm:tukeybound}, where $m$ is the size of the smallest set among $P^-$ and $P^+$. In the worst-case, this nearly matches the upper bound of $\left\lceil\frac{m}{d}\right\rceil$ according to Corollary~\ref{cor:misclassify-upper}.

\section{Experiments}\label{app:detailed_experiments}

For convenience, we provide a detailed self-contained description of the experiments here.  The first paragraph is an exact repetition (included here for convenience). All others are extended versions of the one included in the main text. This version can be read either instead of or next to the shorter version in the main text.

\paragraph{General Settings.} 

We use the GloVe word embeddings~\cite{zhao-glove}\footnote{\texttt{glove.42B.300d}} and limit the dataset to the 150,000 most common words. We create the same classes of male, female and neutral embeddings as \newcite{nullitout} following their approach. A male vector $\vec{M}$ is defined as $\vec{he} - \vec{she}$ and a female vector $\vec{F}$ as $\vec{she} - \vec{he}$. Our male dataset consists of the 7500 data points closest to $\vec{M}$ and our female dataset to the 7500 data points closest to $\vec{F}$. For the neutral dataset we use Ravfogel et al.'s \citeyear{nullitout} seed to obtain the same random selection of 7500 words with a cosine similarity of less than 0.3 to $\vec{M}$. 

We include experiments on two classes (\emph{feminine} and \emph{masculine}) and on three classes (\emph{feminine}, \emph{masculine} and \emph{neutral}). The INLP method uses the datasets to train SVM classifiers as described above. For the MP method, let $\vec{v_F}$, $\vec{v_M}$ and $\vec{v_N}$ be the mean of respectively feminine, masculine and neutral labeled points of a given set $D$. Then in experiments on two classes the mean projection method uses the $\vec{v_F} - \vec{v_M}$ vector for projection. For experiments with all three classes, we use both $\vec{v_N} - \vec{v_F}$ and $\vec{v_N} - \vec{v_M}$ vectors for projection.

\paragraph{Linear Guarding.}

The first set of experiments illustrates the process of gender information being removed from embeddings using each method. We apply the INLP and MP algorithms to the data labeled feminine, masculine and neutral described above and investigate to what extent a linear classifier can still identify the original gender of the word. The results of this classifier should decrease as gender information is removed from the set. For INLP, we use a L2-regularized SVM with the same parameters and random seeds as \citet{nullitout}. We can observe changes over the course of performing multiple projections for INLP. For MP, only a single projection is needed. 

We test the classifier on the same train, test and development set as \citet{nullitout}. The dataset contains three classes (feminine, maschuline and neutral). Like \citet{nullitout}, we report the results for all three classes as well as two classes (male and female). Since all classes are equally distributed, we report on the classifiers accuracy. A set of completely debiased embeddings should yield a result that is equal to chance.

Figure~\ref{fig:acc_compare_three_classes_app} presents the results of 35 iterations of INLP and the single iteration of MP on three classes of the data. Before applying projections, a linear classifier achieves perfect classification. The mean projection method brings down the accuracy to 34.18\%. It takes INLP 12 iterations to reach a comparable result of 34.9\% or at least 6 iterations to reach 39.9\% accuracy. We repeat the experiment on two classes (masculine and feminine) and provide the results in Figure~\ref{fig:acc_compare_two_classes_app}. Here after applying the MP method, an accuracy of 50.6\% is reached while INLP needs 14 iterations to get down to 50.51\% accuracy. These results show that with MP, one projection is enough whereas multiple projections are needed to achieve a similar result with INLP. The results of INLP stabilizes at approximately the same result as MP's single projection.

\begin{figure}
    \centering
    \includegraphics{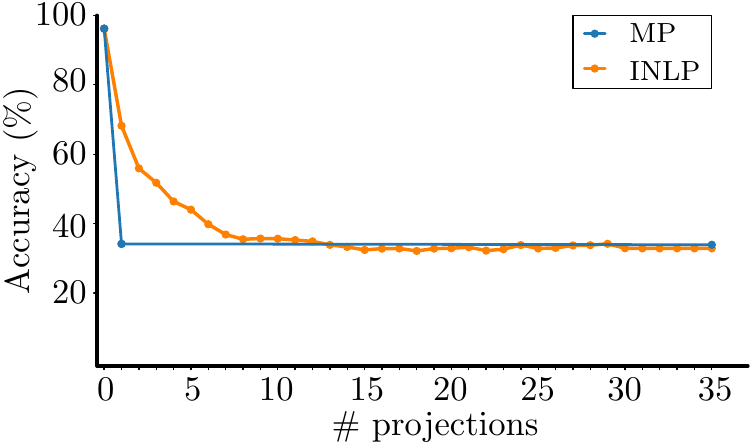}
    \caption{Comparing accuracy of INLP and MP on \textbf{three} classes of feminine, masculine and neutral words.}
    \label{fig:acc_compare_three_classes_app}
\end{figure}

\begin{figure}
    \centering
    \includegraphics{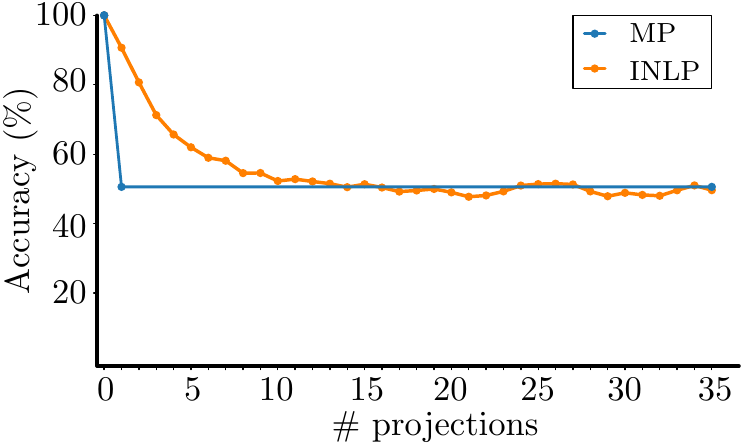}
    \caption{Comparing accuracy of INLP with MP on \textbf{two} classes of feminine and masculine words.}
    \label{fig:acc_compare_two_classes_app}
\end{figure}

\paragraph{Classification.} Both methods only aim to remove linearly encoded information. We compare to what extent they leave non-linear encoding of gender in tact by running a 1-hidden-layered MLP with ReLU activation. This MLP manages to identify gender with 85.0\% after 35 iterations. After a single MP projection, accuracy of the MLP drops to 81.6\%. Overall these results confirm that MP needs a single projection to achieve equal linear guarding as multiple INLP projections and a comparable (slightly better) result on the remaining non-linear encoding.

\paragraph{Clustering.} 
Another way of assessing the debiasing methods is to test to what extent the embeddings can still form clusters representative of the male and female class. If the bias has been removed successfully, the clusters should no longer reflect gender bias information. We use the V-measure \cite{rosenberg2007v} to assess to what extent the two classes are intertwined. This measure is the harmonic mean of two components: \emph{homogeneity} (reflecting how many points in a cluster belong to the same class) and \emph{completeness} (reflecting how many points of a class ended up in the same cluster). The homogeneity component assesses to what extent samples of a single cluster have the same class i.e.\ are similar. Completeness on the other hand measures to what extent data points with the same labels are in the same cluster. A perfect V measure (score 1.0) indicates perfect clustering with respect to the gold labels, where equal distribution of class members from all classes over all clusters yields a score of 0. We apply $K$-means clustering with $K=2$ to the male and female gendered words in the test split, then we calculate the V-measure.

We use the 2000 most biased female and male words as \citet{nullitout} suggests. Using the MP method reduces the V-measure to 0.53\% after a single projection. The V-measure reaches 0.67\% after applying the INLP method with all 35 iterations on the this set.


\paragraph{Exploring Neighbors.} We inspect changes in the direct environment of 40 words randomly selected by \newcite{nullitout} as well as their selection of 22 gendered given names. If a debiasing method is targeted, it should not affect the neighbors of the randomly selected words and only lead to changes in the neighborhood of the words that should be affected by debiasing (in this case the 22 gendered given names, where new name of the opposite gender may appear). 

When applying MP, only 8 out of 120 neighbors of random words changed, compared to 43 after changes caused by INLP. For given names, MP yielded 24 changes out of 66 compared to 56 for INLP. Overall, it seems that INLP leads to more changes in the neighborhoods of words targeted by debiasing as well as other words. 

Upon manual inspection, we observed that MP mainly leads to names being replaced by other given names in the neighborhood of gendered names (14 times,  5 times a name of opposite gender). INLP's new contexts contain more common nouns and rare non-given names (46 tokens that end up in the immediate context of a gendered name are not regularly spelled given names). Another noticeable difference is that MP appears to have moved the name \textit{Ariel} to its interpretation of a location, whereas it ended up near other terms related to Disney after applying INLP. The full results can be found in Tables~\ref{sec:embspacechange} and~\ref{tab:emb_space_gendered_name} in Appendix~\ref{sec:embspacechange}.

\paragraph{Similarity Experiments}

Word debiasing methods should only affect information that reflects the bias and leave the rest of the space unchanged. We investigate the effects of both debiasing methods on three word similarity datasets, namely, simlex-999, Sim353 (both relatedness and similarity dataset) and Mturk-771. The datasets can serve as a reflection of how the debiasing methods affect the semantic space. While it is possible that the similarity datasets contain gendered words, it can be assumed that the majority of word pairs does not have an explicit gendered meaning. We verify this assumption by means of manual analysis of the simlex-999 pairs and found 62 pairs where one or both terms had a meaning that explicitly contains gender. Our manual annotations can be found in our Github repository.
We report the Spearman correlation between the similarity scores given by the annotators and the cosine distance before and after debiasing with INLP and MP in Table~\ref{tab:similarity_correlation}.

\begin{table}[h]
    \renewcommand{\arraystretch}{1.3}
    \footnotesize
    \centering
    \begin{tabu} {X[2.5,l]  X[1.2 l] X[1 l] X[1 l] } \hline
    Dataset & Original & MP  &  INLP\\ \hline 
    simlex-999 & 0.373 & 0.385  &  0.489\\
    WordSim353 - Sim & 0.695 & 0.690 & 0.799 \\
    WordSim353 - Rel & 0.599 & 0.589 & 0.698 \\
    Mturk-771 & 0.684 & 0.691 & 0.728  \\
    \end{tabu}
    \caption{ The correlation of the cosine distances with the similarity scores in the Original space, after MP debiasing and after INLP debiasing. }
    \label{tab:similarity_correlation}
\end{table}

 The results on these similarity tests remain relatively stable after applying MP, whereas we see a clear increase in correlation sores after INLP is applied. This result confirms that the overall embeddings change less when applying MP compared to INLP. The observed improvement caused by INLP indicate that more research is needed to determine how INLP affects the entire semantic space. If iteractive projections can improve the overall quality, the additional effect they have on the space may actually be desirable. In particular, \citet{nullitout} hypothesize that this increase may be due to words in the datasets containing a gender component which is not perceived by humans, but do not investigate the matter further. We further explore this possibility in our additional experiments described in Section~\ref{ssec:divingdeeper} (see main content above).

\paragraph{WEAT.}
\citet{weat} propose the Word Embedding Association Test (WEAT) that can measure bias inferred by the semantic similarity between groups of words. We have two groups of target words for instance $T_1=\{$ programmer, professor, engineer, \dots$\}$, $T_2=\{$nurse, teacher, librarian , \dots $\}$ and two set of attribute words $A_1=\{$ man, male, masculine, \dots $\}$, $A_2=\{$ woman, female, feminine, \dots $\}$. If the word embeddings representing these target words are not biased, then the relative similarities between the target words ($T_1$, $T_2$) and attributes ($A_1$, $A_2$) should be equivalent. In other words: words such as \textit{secretary} and \textit{programmer} should not be more similar to either of the two groups of gendered attributes. 

WEAT measures the association between the target groups and attributes: a biased representation is expected to maintain this association (reflected by a high score), whereas WEAT scores on representations with low to no bias should be close to zero. A WEAT score of 0 indicates that words from the target groups are, on average, equally strongly associated with $A_1$ and $A_2$. Like \newcite{nullitout}, we follow \newcite{lipstick} and represent the attributes through gendered names rather than attributes. We then take the target words of the same
 the three tests used by \citet{nullitout} (\emph{WEAT 6}, \emph{WEAT 7} and \emph{WEAT 8}) in Table~\ref{tab:weat_scores}.\footnote{\citet{nullitout} report p-values of associations rather than the values. We present the values instead, because the p-values do not distinguish between positive and negative WEAT scores.} \emph{WEAT~(6)} includes targets words referring to career and family attributes, \emph{WEAT~(7)} has target words related to math and art and \emph{WEAT~(8)} has science and art target words.

 Both MP and INLP reduce the WEAT scores. For \emph{WEAT 6} and \emph{7}, 35 INLP projections lead to the largest reduction. For \emph{WEAT 8}, MP provides the best results followed by a single INLP iteration (rather than the full 35). The negative value of \emph{WEAT 8} indicates a stronger bias in opposite direction that the remaining bias present after the first INLP iteration or after the MP projection. This leads us to wonder how to interpret these results. We therefore investigate them further in Section~\ref{ssec:divingdeeper}.

\begin{table}[h]
    \centering
    \begin{tabular}{c c c}
    \hline \hline
    \textbf{Original} & &\\ \hline \hline 
        \emph{WEAT 6} & \emph{WEAT 7} & \emph{WEAT 8} \\ \hline
         0.738 & 0.584 & 0.467\\
         \hline \hline
    \textbf{INLP-35} & &\\ \hline \hline
        \emph{WEAT 6} & \emph{WEAT 7} & \emph{WEAT 8} \\ \hline
        0.008 (-98\%) & 0.084 (-86\%) & |-0.310| (-34\%) \\
    \hline \hline
     \textbf{INLP-1} & &\\ \hline \hline
        \emph{WEAT 6} & \emph{WEAT 7} & \emph{WEAT 8} \\ \hline
             0.278 (-62\%) & 0.347 (-41\%) & 0.203 (-56\%) \\
    \hline \hline
    \textbf{MP} & &\\ \hline \hline
    \emph{WEAT 6} & \emph{WEAT 7} & \emph{WEAT 8} \\ \hline
          0.173 (-76\%) & 0.191 (-66\%) & 0.110 (-76\%) \\
    \end{tabular}
    \caption{WEAT scores of tests \emph{WEAT 6}, \emph{WEAT 7}, and \emph{WEAT 8} of the original embeddings, after INLP debiasing with 35 iterations and after MP debiasing. For each score, you can see the decrease from the original WEAT scores.}
    \label{tab:weat_scores}
\end{table}

\paragraph{Bias-by-neighbor.}

\newcite{lipstick} show that even if a debiasing method manages to remove a biased vector from explicitly gender words, it may still be close to other words with the same stereotypical connotation (e.g.\ \emph{nurse} remains close to \emph{receptionist} and \emph{teacher}). The \textbf{bias-by-neighbors} measure captures such remaining bias by providing the percentage of the 100 closest words that were male-biased in the original dataset. For completely debiased word embeddings, this percentage should be around 50\% (with half of the words being (slightly) biased to one gender and the other half of the words to the other).
This measure is applied to a set of biased professions provided by \citet{bolukbasi}, which originally has a bias score of 85.2\%. INLP reduces this score to 73.4\% and MP to 74.5\%. Both methods thus yield comparable results that reveal that bias is reduced but not completely removed.

\paragraph{TPR-GAP: ``gender in the wild''.} \citet{biasbios} propose an approach that considers the impact of bias on downstream classification tasks. The approach measures to what extent debiasing methods can remove gender bias from a system that predicts a person's occupation based on a biography describing them. A fair system should have equal performance for members of different classes \cite{hardt2016equality}. Specifically, it should not amplify a bias that is present in the label distribution. In the case of gender, the system should perform equally well predicting stereotypically gendered occupations such as surgeon, caretaker, secretary or marine, regardless of whether the biographical texts are about men or women. \citet{biasbios} use the $GAP^{TPR}$ score to measure to what degree the performance of a classification system is impacted by bias. The score is calculated as follows: TPR$_{g, y}$ (Equation~\ref{eq:tpr}) is the true positive rates for a given profession ($y$) and gender ($g$). TPR-GAP  Gap$_{g,y}^{\text{TPR}}$ is the difference (gap) between true positive rates (TPR) of gender $g$ and its opposite gender $\sim g$ in a given occupation $y$ (Equation~\ref{eq:gap}).  In order to calculate a single measure for bias in all occupations \citet{romanov} take the root-mean square of the TPR-GAPs for all occupations (Equation~\ref{eq:gap_rms}).


\begin{equation}
    \text{TPR}_{g, y} = P [\hat{Y} = y |G = g,Y = y]
    \label{eq:tpr}
\end{equation}

\begin{equation}
    \text{Gap}^{\text{TPR}}_{g, y} = \text{TPR}_{g,y} - \text{TPR}_{\sim g, y}
    \label{eq:gap}
\end{equation}

\begin{equation}
    \text{GAP}_g^{\text{RMS}} = \sqrt{\frac{1}{|C|} \sum_{y \in C}{{(\text{GAP}_{g, y} ^ {\text{TPR}}} )}^2 }
    \label{eq:gap_rms}
\end{equation}

We obtain 393,423 biographies from \citet{nullitout} which is a subset of the  original corpus of \citet{biasbios} describing people with 28 different occupations.\footnote{\citet{biasbios} use 399,000 biographies. The reduction is due to biographies no longer being available when scraping the data.} We split the data in the same 65\% training, 10\% development and 25\% test splits as used by \citet{biasbios} and \citet{nullitout}.
Like~\citet{nullitout}, we use logistic classifiers that take one of three representations of the biographies as input: (1) one-hot BOW, (2) averaged FastText embeddings \citep{fasttext} and (3) the last hidden state of BERT over the \emph{[CLS]} token. 

In this setup, debiasing projections are applied as as follows: For INLP, each of the three representations is debiased by training a linear classifier on all 28 occupation classes and using the learned vector to calculate the nullspace. Logistic regression in 100 iterations is used for the BOW representations,  150 linear SVM iterations for FastText representations and 300 linear SVM iterations for BERT. We use Scikit Learn \cite{pedregosa2011scikit} for all classification approaches. For the MP setup, we create a mean projection for every occupation $o$ (28 in total) by identifying the projection vector $P_o$ based on the mean of all female and the mean of all male biographies with occupation $o$.

 We report the accuracy of predicting the correct profession (Acc.), the root-mean-square of the $GAP^{TPR}$ scores of all professions ($GAP^{RMS}$) and the correlation between percentage of women in a given profession and the respective $GAP^{TPR}$ scores (Corr.). If debiasing is successful, the classification accuracy should remain high (compared to before debiasing), while the GAP score should decrease. 
We measure the correlation between underpredicting women (the $GAP^{TPR}$ score) and the true percentage of women in the set of professions. A lower correlation is preferable. The intuition behind measuring the correlations in this way is to measure the degree to which bias is amplified compared to the distribution of men and women in a profession. For example, a difference of 5\% can have different implications depending on the original distribution: It is worse if the predicted percentage of females in a profession goes down to 15\% when the true percentage is 20\% compared to a true percentage of 48\% going to 43\%. 

The results presented in Table~\ref{tab:biography_results} show that overall results of INLP and MP are comparable. MP retains higher accuracy compared to INLP. MP reduces the $GAP^{RMS}$ much more for one-hot BOW, slightly more for FastText embeddings and slightly less for BERT. INLP reduces the correlation a bit more for both FastText and BERT.

\begin{table}[h]
    \centering
    \renewcommand{\arraystretch}{1.2}
    \begin{tabular}{l c c c}
        \hline \hline
        \textbf{Original} & &  &  \\ \hline \hline
            Representation & Acc. &  $\text{GAP}_\text{female}^{\text{RMS}}$ &  Corr. \\ \hline
            one-hot BOW &    77.6\%&      0.182 &    0.895 \\
            FastText &       77.9\% &     0.172 &    0.891\\
            BERT &           74.7\% &     0.153 &    0.852\\
            
        \hline \hline
        \textbf{INLP} & &  &  \\ \hline \hline
            Representation & Acc. &  $\text{GAP}_\text{female}^{\text{RMS}}$ &  Corr. \\
            \hline
            one-hot BOW &   77.1     &   0.111 &       0.65 \\
            FastText &      73.6\% &     0.103 &    0.493\\
            BERT &          74.7\% &     0.065 &    0.353\\
            \hline \hline
            \textbf{MP} & &  &  \\ \hline \hline
            Representation &          Acc. &  $\text{GAP}_\text{female}^{\text{RMS}}$ &  Corr. \\
            \hline
            one-hot BOW &   76.67&        0.001 &       0.58 \\
            FastText &      75.6\%&  0.092 &    0.522\\
            BERT &      75.2\% &     0.087 &    0.395\\
            
    \end{tabular}
    \caption{We report on accuracy of the main task which is the occupation classification (Acc.), the TPR-GAP bias measure ($\text{GAP}_\text{female}^{\text{RMS}}$) and the correlation (Corr.) of TPR$_{g}$ and percentage of women for each occupation. You can see the report per vector representation after INLP and MP debiasing method as well as the original representations without debiasing.}
     \label{tab:biography_results}
\end{table}

\newpage

\section{Additional Results of Diving Deeper}\label{app:diving_even_deeper}

This appendix provides additional results on the WEAT scores. Figure~\ref{fig:inlp_weat_random} presents the result of applying random projection iterations after 8 INLP iterations. Figure~\ref{fig:weat_random_app} illustrates when random projections are applied to the original data (without first applying projections for removing gender). These figures illustrate that random projections after 8 INLP iterations mostly lead to a decrease in WEAT score, whereas applying random projection directly to fully biased data increases the scores in the far majority of the cases.

\begin{figure*}[ht!]
\begin{minipage}[b]{\columnwidth}
    \centering
    \includegraphics{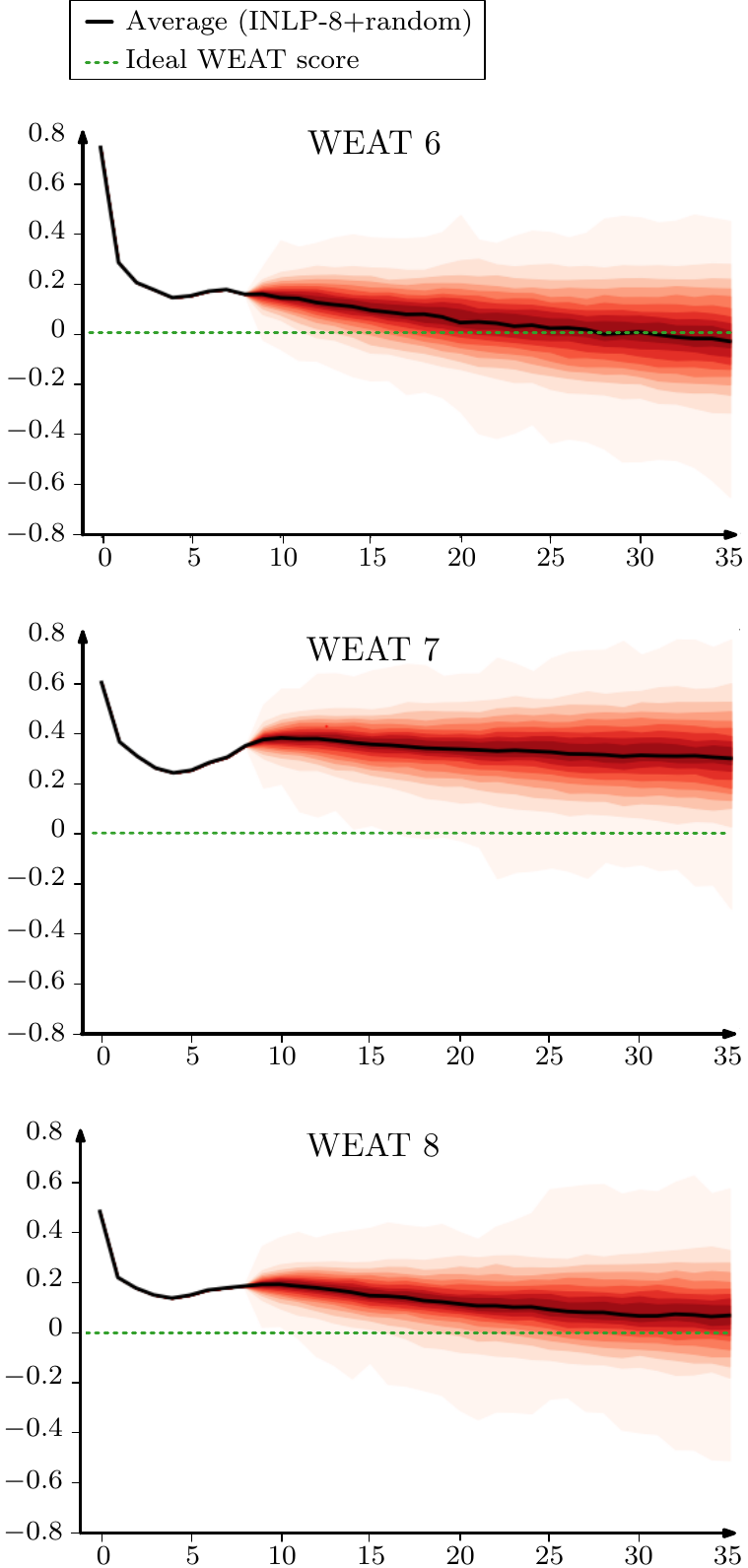}
    \caption{Plot of 500 runs where we first apply INLP 8 times and then continue with 27 iterations of random projections. The shades of red indicate various levels of confidence in WEAT scores ($y$-axis) after a certain number of iterations ($x$-axis), where darker red shades reflect higher probabilities of the WEAT score value after a certain number of iterations. The black line is the average WEAT score of all runs and the dotted green line is drawn at $y$=0 (the ideal WEAT score).}
    \label{fig:inlp_weat_random}
\end{minipage}
\hfill
\begin{minipage}[b]{\columnwidth}
    \centering
    \includegraphics{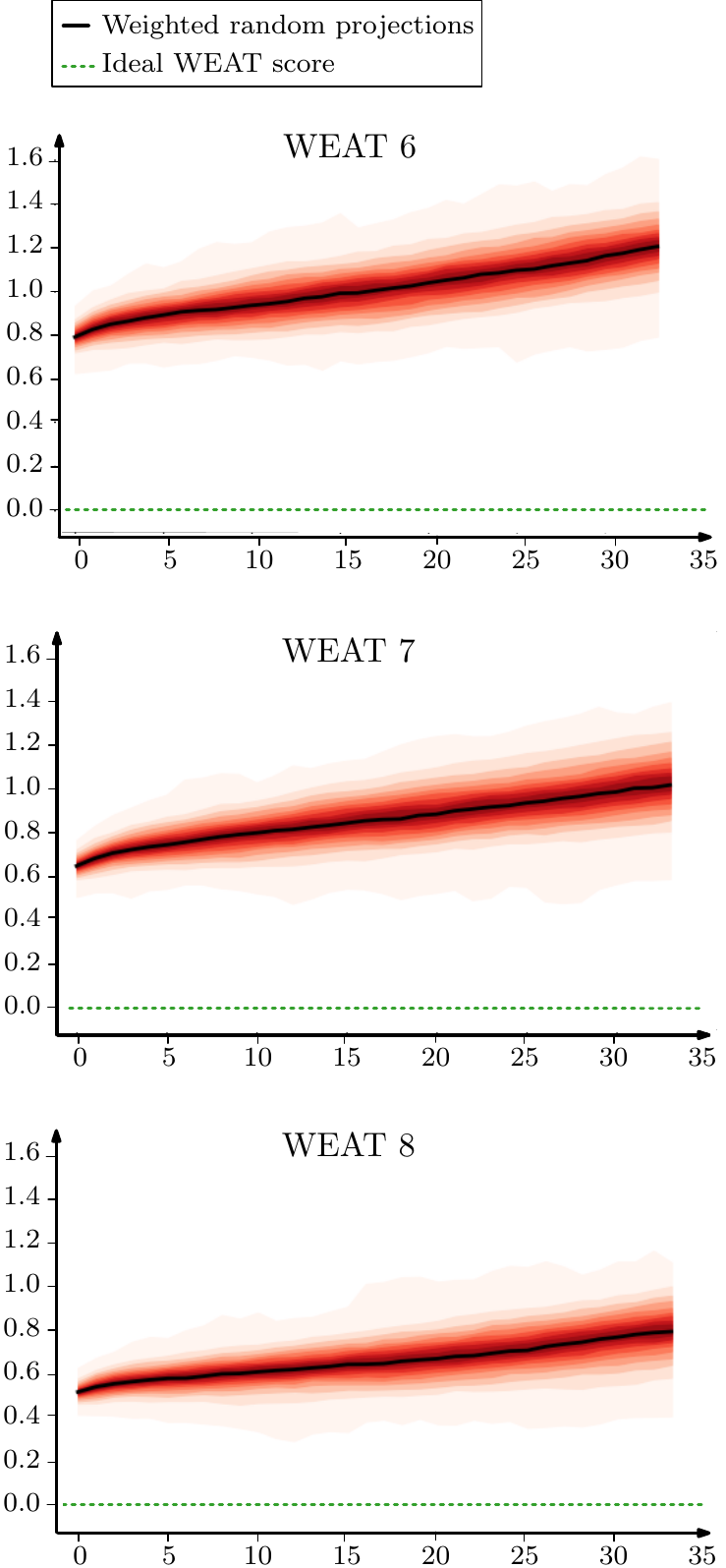}
    \caption{Plot of 500 runs where in every iteration we perform a weighted random projection. The shades of red indicate various levels of confidence in WEAT scores ($y$-axis) after a certain number of iterations ($x$-axis), where darker red shades reflect higher probabilities of the WEAT score value after a certain number of iterations. The black line is the average WEAT score of all runs and the dotted green line is drawn at $y$=0 (the ideal WEAT score).}
    \label{fig:weat_random_app}
\end{minipage}
\end{figure*}

\newpage

\onecolumn
\section{Changes in the Embedding Space of GloVe}
\label{sec:embspacechange}

\begin{table*}[h]
    \renewcommand{\arraystretch}{1.3}
    \centering
    \footnotesize
    \begin{tabu} {X[1.1,l]  X[3 l] X[3 l] X[3, l] } \hline
        Word & Original neighbors & Neighbors after MP & Neighbors after INLP  \\ \hline 
        order & orders, ordering, purchase & orders, ordering, purchase &  orders, ordering, \textcolor{CBred}{ordered} \\
        crack & keygen, cracks, torrent & keygen, cracks, torrent &  keygen, cracks, \textcolor{CBred}{warez} \\
        craigslist & ebay, craiglist, ads & ebay, craiglist, ads &  ebay, craiglist, \textcolor{CBred}{freecycle}\\
        populations & population, species, communities & population, species, communities & population, species, \textcolor{CBred}{habitats} \\
        epub & ebook, mobi, pdf & ebook, mobi, pdf & mobi, ebook, \textcolor{CBred}{kindle} \\
        finals & semifinals, playoffs, championship & semifinals, playoffs, semifinal & semifinals, semifinal, \textcolor{CBred}{quarterfinals}\\
        installed & install, installing, installation & install, installing, installation & install, installing, \textcolor{CBred}{installs}\\
        identifiable & disclose, identify, identifying & disclose, identify, identifying & disclose, \textcolor{CBred}{pii}, \textcolor{CBred}{distinguishable}\\
        photographs & photograph, photos, images & photograph, photos, images & photograph, images, photos\\
        ta & si, tu, ti & si, tu, \textcolor{CBred}{na} & \textcolor{CBred}{que}, \textcolor{CBred}{bien}, \textcolor{CBred}{ele}\\
        couch & sofa, sitting, bed & sofa, sitting, bed & sofa, \textcolor{CBred}{couches}, \textcolor{CBred}{loveseat} \\
        cooler & coolers, cooling, warmer & coolers, cooling, warmer & coolers, cooling, warmer\\
        becky & debbie, kathy, julie & debbie, kathy, \textcolor{CBred}{karen} & debbie, \textcolor{CBred}{steph}, \textcolor{CBred}{jen}\\
        appreciated & appreciate, greatly, thanks & appreciate, greatly, thanks & appreciate, \textcolor{CBred}{muchly}, thanks\\
        negotiation & negotiating, negotiations, mediation & negotiating, negotiations, mediation & negotiating, negotiations, mediation\\
        initial & subsequent, prior, following & subsequent, prior, following &  \textcolor{CBred}{intial}, \textcolor{CBred}{inital}, \textcolor{CBred}{subsequent}\\
        chloe & chanel, emma, lauren & chanel, \textcolor{CBred}{handbags}, \textcolor{CBred}{handbag} & chloe, chanel, \textcolor{CBred}{handbags} \\
        filipino & pinoy, filipinos, philippine & pinoy, filipinos, philippine &  filipinos, pinoy, \textcolor{CBred}{tagalog}\\
        relying & rely, relied, relies & rely, relied, relies & rely, relied, relies\\
        perpetual & eternal, continual, irrevocable & eternal, continual, irrevocable & irrevocable, \textcolor{CBred}{datejust}, \textcolor{CBred}{perpetuity} \\
        himself & him, herself, his & herself, him, \textcolor{CBred}{myself} & herself, \textcolor{CBred}{oneself}, \textcolor{CBred}{he}\\
        seaside & beach, beachside, picturesque & beach, beachside, picturesque & beachside, \textcolor{CBred}{idyllic}, \textcolor{CBred}{seafront}\\
        measure & measures, measuring, measured & measures, measuring, measured & measures, measuring, measured\\
        yorkshire & staffordshire, leeds, lancashire & staffordshire, leeds, lancashire & staffordshire, \textcolor{CBred}{dales}, lancashire\\
        merchandise & goods, items, apparel & goods, items, apparel & goods, items, \textcolor{CBred}{merchandize} \\
        sub & subs, k, def & subs, def, k & subs, \textcolor{CBred}{subbed}, \textcolor{CBred}{svs} \\
        tones & tone, hues, muted & tone, hues, muted &  tone, \textcolor{CBred}{polyphonic}, muted\\
        therapist & therapists, psychologist, therapy & therapists, therapy, psychologist &  therapists, \textcolor{CBred}{physiotherapist}, psychologist\\
        leaned & sighed, smiled, glanced & sighed, \textcolor{CBred}{leant}, smiled & \textcolor{CBred}{leant}, \textcolor{CBred}{leaning}, sighed\\
        tho & nnd, cuz, tlie & nnd, cuz, tlie & nnd, \textcolor{CBred}{tlio}, tlie\\
        lawyers & attorneys, lawyer, attorney & attorneys, lawyer, attorney &  attorneys, lawyer, attorney \\
        compile & compiling, compiler, compiles & compiling, compiles, compiler &  compiling, compiler, compiles \\
        chord & chords, progressions, guitar & chords, progressions, \textcolor{CBred}{melody} &  chords, progressions, \textcolor{CBred}{voicings} \\
        aims & aim, aimed, aiming & aim, aimed, aiming &  aim, aimed, aiming\\
        ensure & ensuring, assure, ensures & ensuring, assure, ensures &  ensuring, ensures, assure\\
        aerospace & aviation, engineering, automotive & aviation, engineering, automotive &  \textcolor{CBred}{aeronautics}, aviation, \textcolor{CBred}{aeronautical} \\
        clubhouse & pool, playground, amenities & pool, playground, amenities & \textcolor{CBred}{clubhouses}, pool, playground\\
        locking & lock, locks, latch & lock, locks, latch &  lock, locks, latch\\
        reign & reigns, emperor, throne & reigns, throne, \textcolor{CBred}{reigned} & reigns, \textcolor{CBred}{reigned}, emperor\\
        vulnerable & susceptible, fragile, affected & susceptible, fragile, affected & susceptible, \textcolor{CBred}{vunerable}, fragile\\

    \end{tabu}
   \caption{The 3-nearest neighbors of 40 randomly chosen words are shown originally, after mean projection (MP) and after INLP debiasing method. Changed tokens are marked in red.}
    \label{tab:emb_space_random}
\end{table*}


\begin{table*}[t]
    \renewcommand{\arraystretch}{1.3}
    \footnotesize
    \centering
    \begin{tabu} {X[1.1,l]  X[2.5 l] X[2.5 l] X[3.5 l] } \hline
    Word & Original neighbors & Neighbors after MP & Neighbors after INLP\\ \hline 
    ruth & helen, esther, margaret & esther, helen, margaret & \textcolor{CBred}{etting}, esther, \textcolor{CBred}{gehrig} \\
    charlotte & raleigh, nc, atlanta & raleigh, nc, atlanta & raleigh, \textcolor{CBred}{greensboro}, nc \\
    abigail & hannah, lydia, eliza & hannah, lydia, \textcolor{CBblue}{josiah} & hannah, \textcolor{CBred}{phebe}, \textcolor{CBblue}{josiah}\\
    sophie & julia, marie, lucy & \textcolor{CBred}{claire}, julia, \textcolor{CBred}{madeleine} & \textcolor{CBred}{moone}, \textcolor{CBred}{bextor}, \textcolor{CBred}{marceau}\\
    nichole & nicole, kimberly, kayla & nicole, kimberly, \textcolor{CBred}{mya} & nicole, \textcolor{CBred}{mya}, \textcolor{CBred}{heiress}\\
    emma & emily, lucy, sarah & emily, \textcolor{CBred}{watson}, sarah & \textcolor{CBred}{grint}, \textcolor{CBred}{frain}, \textcolor{CBred}{watson}\\
    olivia & emma, rachel, kate & \textcolor{CBred}{wilde}, \textcolor{CBred}{munn}, \textcolor{CBblue}{oliver} & \textcolor{CBred}{munn}, \textcolor{CBred}{thirlby}, \textcolor{CBred}{wilde}\\
    ava & devine, zoe, isabella & devine, isabella, \textcolor{CBred}{appellation} & \textcolor{CBred}{viticultural}, \textcolor{CBred}{devine}, \textcolor{CBred}{appellation}\\
    isabella & sophia, josephine, isabel & isabel, josephine, \textcolor{CBblue}{henry} & \textcolor{CBred}{rossellini}, \textcolor{CBred}{beeton}, \textcolor{CBblue}{ferdinand}\\
    sophia & anna, lydia, julia & \textcolor{CBred}{hagia}, \textcolor{CBred}{antipolis}, sofia & \textcolor{CBred}{hagia}, \textcolor{CBred}{antipolis}, \textcolor{CBred}{topkapi}\\
    mia & bella, mamma, mama & bella, mamma, \textcolor{CBred}{che} & \textcolor{CBred}{bangg}, mamma, \textcolor{CBred}{culpa}\\
    amelia & earhart, louisa, caroline & earhart, \textcolor{CBred}{fernandina}, louisa & earhart, \textcolor{CBred}{fernandina}, \textcolor{CBred}{bedelia}\\
    james & john, william, thomas & john, william, thomas & \textcolor{CBred}{jassie}, \textcolor{CBred}{nightfire}, \textcolor{CBred}{perse}\\
    john & james, william, paul & james, william, \textcolor{CBblue}{mary} & \textcolor{CBred}{deere}, \textcolor{CBred}{scatman}, \textcolor{CBred}{betjeman}\\
    robert & richard, william, james & richard, william, james & \textcolor{CBred}{pattinson}, \textcolor{CBred}{mccammon}, \textcolor{CBred}{blacksportsonline}\\
    michael & david, mike, brian & david, mike, \textcolor{CBred}{jackson} & \textcolor{CBred}{micheal}, \textcolor{CBred}{franti}, \textcolor{CBred}{moorcock}\\
    william & henry, edward, james & henry, edward, \textcolor{CBred}{charles} & edward, henry, \textcolor{CBred}{sir}\\
    david & stephen, richard, michael & richard, \textcolor{CBred}{alan}, stephen & \textcolor{CBred}{bisbal}, \textcolor{CBred}{magen}, \textcolor{CBred}{sylvian}\\
    richard & robert, william, david & robert, william, \textcolor{CBred}{john} & \textcolor{CBred}{clayderman}, \textcolor{CBred}{brautigan}, \textcolor{CBred}{rorty}\\
    joseph & francis, charles, thomas & \textcolor{CBblue}{mary}, francis, charles & \textcolor{CBred}{joesph}, \textcolor{CBred}{dreamcoat}, \textcolor{CBred}{abboud}\\
    thomas & james, william, john & william, james, john & \textcolor{CBred}{szasz}, \textcolor{CBred}{deshaun}, \textcolor{CBred}{tomy}\\
    ariel & sharon, alexis, hanna & \textcolor{CBred}{israel}, \textcolor{CBred}{israeli}, \textcolor{CBred}{gaza} & \textcolor{CBblue}{peterpan}, \textcolor{CBred}{mermaid}, \textcolor{CBred}{cinderella}\\
    mike & brian, chris, dave & dave, \textcolor{CBred}{jim}, chris & \textcolor{CBred}{mignola}, \textcolor{CBred}{birbiglia}, dave\\

    \end{tabu}
    \caption{The 3-nearest neighbors of a selection of inherently gendered names are shown originally, after mean projection (MP) and after INLP debiasing method. Changes to same gender, non-gendered or common nouns are marked in red. Changes to names associated with the opposite gender are marked in blue.}
    \label{tab:emb_space_gendered_name}
\end{table*}

\end{document}